\documentclass[10pt,twocolumn,letterpaper]{article}

\usepackage{cvpr}
\usepackage{times}
\usepackage{epsfig}
\usepackage{graphicx}
\usepackage{amsmath}
\usepackage{amssymb}
\usepackage{amsthm}
\usepackage{amsfonts}       % blackboard math symbols
\usepackage{multirow}
\usepackage{tabularx} % Automatically calculates column widths with text wrapping

\makeatletter
\@namedef{ver@everyshi.sty}{}
\makeatother
\usepackage{pgfplotstable} % Generates table from .csv
\usepackage{booktabs}       % professional-quality tables
\usepackage{url}            % simple URL typesetting
\usepackage{gensymb}

\usepackage[pagebackref=true,breaklinks=true,letterpaper=true,colorlinks,bookmarks=false]{hyperref}

\makeatletter\let\captiontemp\@makecaption\makeatother
\usepackage[font=footnotesize,labelformat=simple]{subcaption}
\makeatletter\let\@makecaption\captiontemp\makeatother
\graphicspath{{./figures/}}

\input{local-definitions}

\newcommand{\tr}{\mbox{$^{\top}$}}

\def\R{{\rm I} \! {\rm R}}

\newcommand{\eqn}[1]{(\ref{eqn:#1})}
\newcommand{\fig}[1]{Figure~\ref{fig:#1}}
\newcommand{\tab}[1]{Table~\ref{tab:#1}}
\newcommand{\Tab}[1]{Table~\ref{tab:##2}}

\newcommand{\Sect}[1]{Section~\ref{sec:#1}}

\newcommand{\SKIP}[1]{} % Used to skip stuff we do not want type-set

\newcommand{\mbegin} {\left [ \begin{array}}

\newcommand{\mend}   {\end{array} \right ]}

\newcommand{\detbegin} {\left | \begin{array}}
\newcommand{\detend}   {\end{array} \right |}

\newcommand{\vbegin} {\left ( \begin{array}{c}}

\newcommand{\vend} {\end{array}\right )}

\def\squareforqed{\hbox{\rlap{$\sqcap$}$\sqcup$}}

\def\qed{\ifmmode\squareforqed\else{\unskip\nobreak\hfil
	\penalty50\hskip1em\null\nobreak\hfil\squareforqed
	\parfillskip=0pt\finalhyphendemerits=0\endgraf}\fi}

\def\vec#1{\mathchoice%
	{\mbox{\bf $\displaystyle\boldsymbol{\bf{#1}}$}}
	{\mbox{\bf $\textstyle\boldsymbol{\bf{#1}}$}}
	{\mbox{\bf $\scriptstyle\boldsymbol{\bf{#1}}$}}
	{\mbox{\bf $\scriptscriptstyle\boldsymbol{\bf{#1}}$}}}
\def\v#1{\protect\vec #1}

\newcommand{\showeqnlabel}{
	\hbox to 0pt{\quad\quad\relax\fbox{\scriptsize\rm\eqnlblx}%
	\gdef\eqnlblx{xxxx}}} \newcommand{\eqnlblx}{}

\def\@eqnnum{\rm (\theequation)\showeqnlabel}

\newcommand{\nofig}[1]{\centerline{\bf Figure here}}

\def\mat#1{\mathchoice{\mbox{\bf$\displaystyle\tt#1$}}
	{\mbox{\bf$\textstyle\tt#1$}}
	{\mbox{\bf$\scriptstyle\tt#1$}}
	{\mbox{\bf$\scriptscriptstyle\tt#1$}}}
\def\m#1{\protect\mat #1}

\def\thmcolon{\relax}

\newtheorem{THEOREM}{Theorem}
\newtheorem{LEMMA}[THEOREM]{Lemma}
\newtheorem{PROPOSITION}[THEOREM]{Proposition}
\newtheorem{COROLLARY}[THEOREM]{Corollary}
\newtheorem{DEFINITION}[THEOREM]{Definition}
\newtheorem{OBSERVATION}[THEOREM]{Observation}
\usepackage{xpatch}
\makeatletter
\AtBeginDocument{\xpatchcmd{\@thm}{\thm@headpunct{.}}{\thm@headpunct{:}}{}{}}
\makeatother

\newenvironment{lemma}{\begin{LEMMA} \thmcolon \rm}{\end{LEMMA}}

\newcommand{\Dh}[1]{{\tilde{\v #1}}}

\cvprfinalcopy % *** Uncomment this line for the final submission

\ifcvprfinal\fi
\begin{document}

%%%%%%%%% TITLE
\title{Joint Unsupervised Learning of Optical Flow and Egomotion with Bi-Level Optimization}

\author{%
Shihao Jiang\textsuperscript{1,2,3}, Dylan Campbell\textsuperscript{1,2}, Miaomiao Liu\textsuperscript{1,2}, Stephen Gould\textsuperscript{1,2}, Richard Hartley\textsuperscript{1,2}\\
\textsuperscript{1}Australian National University, %
\textsuperscript{2}Australian Centre for Robotic Vision, %
\textsuperscript{3}Data61/CSIRO %
% {\tt\{email.address\}@anu.edu.au}%
}

\maketitle
%\thispagestyle{empty}

%%%%%%%%% ABSTRACT
\begin{abstract}
We address the problem of joint optical flow and camera motion estimation in rigid scenes by incorporating geometric constraints into an unsupervised deep learning framework. 
Unlike existing approaches which rely on brightness constancy and local smoothness for optical flow estimation, we exploit the global relationship between optical flow and camera motion using epipolar geometry. 
In particular, we formulate the prediction of optical flow and camera motion as a bi-level optimization problem, consisting of an upper-level problem to estimate the flow that conforms to the predicted camera motion, and a lower-level problem to estimate the camera motion given the predicted optical flow. 
We use implicit differentiation to enable back-propagation through the lower-level geometric optimization layer independent of its implementation, allowing end-to-end training of the network. 
With globally-enforced geometric constraints, we are able to improve the quality of the estimated optical flow in challenging scenarios, and obtain better camera motion estimates compared to other unsupervised learning methods. 
\end{abstract}

%%%%%%%%% BODY TEXT
\section{Introduction}
\label{sec:introduction}
Dense optical flow estimation is a fundamental problem in computer vision for determining the apparent motion of pixels in an image as the camera and scene moves. It has broad applications in action recognition~\cite{simonyan2014two}, 3D reconstruction \cite{kumar2017monocular}, and motion segmentation~\cite{narayana2013coherent}.
The seminal work by Horn and Schunck~\cite{horn1981determining} sets the foundation for solving optical flow estimation problems by enforcing brightness constancy and local smoothness constraints in a variational setting. In the last few decades, the quality of optical flow estimation has improved dramatically with the introduction of ideas such as piece-wise smoothness~\cite{black1996estimating}, coarse-to-fine refinement for large displacements~\cite{brox2004high}, and layered formulations for handling occlusions~\cite{xiao2006bilateral}.

Like many problems in computer vision, approaches based on the supervised learning of convolutional neural networks (CNNs) now achieve the state-of-the-art results for optical flow estimation~\cite{dosovitskiy2015flownet, ilg2017flownet, sun2018pwc}. However, the difficulty of obtaining large volumes of ground-truth optical flow limits the applicability of these approaches in many scenarios. Unsupervised learning approaches are a promising alternative, which encode brightness constancy and local smoothness constraints in a loss function for training deep networks~\cite{jason2016back}. While such constraints perform well in feature-rich regions, they often fail on challenging scenes with featureless or repetitively-textured regions.

To achieve more robust prediction and handle these cases more effectively, we exploit the geometric constraint between the optical flow and camera motion in an unsupervised learning framework. Specifically, we focus on optical flow estimation in mostly rigid scenes, where optical flow is predominantly caused by camera motion~\cite{sun2014quantitative, wulff2017optical}. Thus the displacement of corresponding pixels across images, i.e.,~the optical flow, satisfies the well-known epipolar constraint~\cite{hartley2004multiple}. We formulate this as an~\emph{epipolar geometric loss} defined on an essential matrix determined from the optical flow. Compared to the fundamental matrix used in existing work~\cite{zhong2019unsupervised}, the essential matrix provides tighter constraints and can be obtained from state-of-the-art geometric algorithms~\cite{hartley2012efficient} that provide more accurate camera motion estimates than, for example, the 8-point algorithm~\cite{hartley2004multiple}.

To compute the epipolar geometric loss we first require an estimate of the camera motion.
As such, we formulate flow estimation as a bi-level optimization problem. The upper-level problem is to estimate the optical flow by minimizing the epipolar geometric loss as well as enforcing the standard brightness constancy constraint. The geometric loss is defined based on an~\emph{essential matrix} encoding of camera motion, which is obtained by solving a lower-level optimization problem that estimates the camera motion from the optical flow. To enable end-to-end training, we use implicit differentiation to back-propagate the gradient of the upper-level loss through the essential matrix estimation layer (the lower-level problem). An overview of our training pipeline is shown in \figref{fig:flowchart}.

Overall, our key technical contributions include: (1) the introduction of a geometric constraint into an unsupervised deep learning framework for optical flow and camera motion estimation, and (2) the formulation of an end-to-end trainable model with an embedded optimization layer that estimates the camera motion required for computing the epipolar loss. Importantly, our formulation allows back-propagation through the optimization layer regardless of the algorithmic implementation used for computing the essential matrix. We show that our geometrically-constrained model can accurately estimate optical flow satisfying the epipolar geometry. Our optical flow estimation method outperforms approaches that ignore geometry and produces remarkably good results on cases that previous approaches find challenging, such as featureless regions and regions with repetitive features. Our camera motion estimation also compares favourably against methods that directly use a network to predict camera poses. 

\begin{figure*}[!t]\centering
\def\svgwidth{\textwidth}
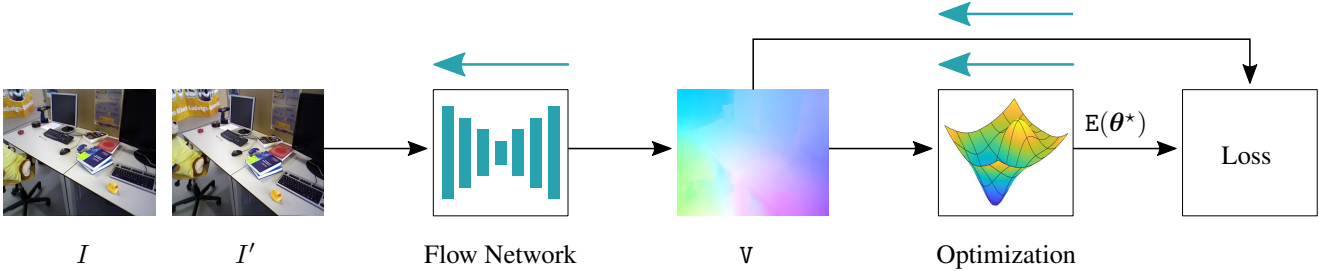
\caption{Overview of our optical flow and egomotion training pipeline. With implicit differentiation, the gradient can be back-propagated through a complex geometric optimization algorithm. The blue-green arrows show the back-propagation direction of the gradients. The input images are denoted by $I$ and $I'$ and the predicted optical flow from the network is denoted by $\m V$. The estimated flow is then fed into a geometric optimization layer and outputs an essential matrix that best fits the optical flow $\m E(\v \theta^*)$, parametrized by $\v \theta^*$. The essential matrix (egomotion) and the predicted flow are used to compute the upper-level loss for training the network.} 
\label{fig:flowchart}
\end{figure*}

\section{Related Work}
\label{sec:related_work}

\textbf{Learning-based optical flow estimation.}
Optical flow estimation has been studied extensively since the pioneering work of Horn and Schunck~\cite{horn1981determining}. The reader is directed to Fortun \etal~\cite{fortun2015optical} and Sun \etal~\cite{sun2018pwc} for comprehensive reviews of the literature. Recent approaches formulate optical flow estimation as a supervised deep learning task. Compared to traditional methods, convolutional neural networks (CNNs) have the advantage of fast inference once trained, making real-time prediction possible. For example, PWC-Net~\cite{sun2018pwc}, building on previous supervised optical flow networks~\cite{dosovitskiy2015flownet,ilg2017flownet}, incorporated traditional optical flow techniques into the network such as cost volumes and feature warping, and achieved better performance than traditional methods with a shorter running time.

Despite strong results, supervised deep learning approaches are limited by the need for ground-truth optical flow during training, which is difficult to obtain for real-world scenes. Unsupervised learning instead allows the network to be trained on large volumes of unlabelled data. Several unsupervised loss functions have been proposed, including photometric constancy and local smoothness losses~\cite{jason2016back}, and occlusion-aware bidirectional consistency and robust census losses~\cite{meister2018unflow}. 
Other works explicitly reason about occlusion~\cite{wang2018occlusion}, or synthetically augment data for better occlusion estimation~\cite{liu2019ddflow, liu2019selflow}.
All of these approaches rely on local matching so cannot handle smooth image regions.
Our work is built on some of these previous works, but also enforces global geometric consistency, allowing it to better handle smooth, featureless regions.

\textbf{Optical Flow and Epipolar Geometry.}
There has been an extensive study on the relationship between optical flow and epipolar geometry and their applications. 
Weber and Malik \cite{weber1997rigid} first applied epipolar geometry to estimate and track independently moving objects from optical flow. 
Early works have also been reviewed in the book by Xu and Zhang~\cite{xu2013epipolar}, which propose to look at the correspondence problem from the standpoint of epipolar geometry. 
Difficult 2-D search problem can be simplified to a 1-D search problem under the assumption that the epipolar geometry is known \textit{a priori}, which is also demonstrated later by Yamaguchi~\etal~\cite{yamaguchi2013robust} on the problem of rigid scene optical flow estimation.
In contrast, rather than treating the estimated epipolar geometry as known \textit{a priori} and imposing a hard constraint, we propose a soft constraint between optical flow and epipolar geometry, and a joint optimization approach between the two in a deep learning context. 
Our idea is similar to~\cite{valgaerts2008variational} in that we couple the estimation of epipolar geometry and optical flow to solve a joint optimization problem. However, we propose a method that can be end-to-end trainable in an unsupervised learning framework. 

Recently, there have been works that leverage epipolar geometry in unsupervised learning framework~\cite{bai2016exploiting, zhong2019unsupervised} with the aim of handling multiple motions. We instead focus on static scenes and demonstrate how to back-propagate the gradients of the loss function through the geometric estimation layer and train the network in an end-to-end manner.  

\textbf{Unsupervised Learning of Camera Motion.} 
Another line of works that has gained popularity in recent years is unsupervised learning of depth and motion from videos since the first work of Zhou \etal \cite{zhou2017unsupervised}. Subsequent works adopted the idea of joint learning of flow, motion and depth and witnessed marginal improvements at the cost of large network parameters \cite{yin2018geonet, zou2018df}. All previous works use a camera motion network to directly regress motion from two images, whereas we propose a hybrid method: asking a network to infer dense correspondences (optical flow), and then use optimization to solve for the camera motion. This hybrid learning idea has also been investigated by Zhou \etal \cite{zhou2019learn} in the context of visual localization.

\textbf{Differentiable optimization.}
A few recent works embed optimization problems as layers within a deep learning model \cite{amos2017optnet, fernando2016learning, lee2019meta, rajeswaran2019meta}. These layers encode complex dependencies and constraints that cannot be easily learned by convolution or fully-connected layers. It is apparently non-trivial to back-propagate gradients through these layers. However,
Gould \etal~\cite{gould2016differentiating, gould2019deep}
addressed this problem and provided general techniques for differentiating argmin and argmax problems (both constrained and unconstrained). Amos \etal~\cite{amos2017optnet} proposed a differentiable optimization layer but was limited to quadratic programs. These techniques relating to differentiable optimization have been used to solve several computer vision problems, such as video classification \cite{fernando2016learning} and meta-learning \cite{lee2019meta, rajeswaran2019meta}.
In this work, we embed an epipolar geometric constraint into a deep learning framework via bi-level optimization, jointly estimating camera motion and optical flow in an unsupervised fashion.

\section{Bi-Level Optimization for Optical Flow and Essential Matrix Estimation}
\label{sec:bilevel}

We define optical flow as a dense field of displacement vectors, where the displacement vector at each pixel coordinate in one image points to the coordinate of the corresponding pixel in another image.
Let a pixel coordinate in image $I \in \R^{W \times H \times 3}$ be denoted by $\v p = (u, v)$ and the corresponding pixel in the other image $I'$ by $\v p' = (u', v')$.
We assume that $I$ and $I'$ are two views of the same scene, typically consecutive frames from a video sequence.
The optical flow between $I$ and $I'$ is then the matrix $\m V = [\v v_1, \dots, \v v_N] \in \R^{2 \times N}$ of displacement vectors $\v v_i = \v p_i' - \v p_i$ for every pixel $\v p_i$ in image $I$ and corresponding pixel $\v p_i'$ in image $I'$. Here $N = WH$ is the total number of pixels in image $I$.

Given the camera intrinsic calibration matrices $\m K$ and $\m K'$ for the image pair $I$ and $I'$, we can obtain the normalized coordinates as $\v x = \m K^{-1} \Dh{p}$ and $\v x' = \m K'^{-1} \Dh{p}'$, where $\Dh{p} = (u, v, 1)\tr$ is the pixel $\v p$ expressed in homogeneous coordinates.
Corresponding points $\v x \leftrightarrow \v x'$ in normalized coordinates satisfy the geometric relationship $\v x'\tr \m E \v x = 0$, known as the \emph{epipolar constraint}.
For camera matrices $\m P = \m K[\m I | 0]$ and $\m P' = \m K'[\m R | \v t]$, the essential matrix can be decomposed into rotation $\m R$ and translation $\v t$ components as
$\m E = [\v t]_{\times}\m R$
with $\v t$ known up to scale~\cite{hartley2004multiple}.
We address the problem of incorporating this epipolar constraint into an optimization procedure for deep optical flow estimation.

\subsection{Optimizing an Epipolar Loss Function}
\label{sec:bilevel_loss}

We formulate optical flow estimation as a bi-level optimization problem, with an upper-level problem that is solved subject to constraints enforced by a lower-level problem, given by
\begin{align}
\displaystyle \mathop{\text{minimize}}_{\m V}
&\enspace L(\m V, \v \theta^{\star}) \label{eqn:bilevel_objective}\\
\text{subject to}
&\enspace \v \theta^{\star} \in \displaystyle \argmin_{\v \theta \in \R^5} l(\m V, \v \theta) \label{eqn:bilevel_constraint}
\end{align}
where $\v \theta \in \R^5$ is the minimal parametrization of $\m E$ given in Hartley \& Li \cite{hartley2012efficient}.
The upper-level loss function $L$ comprises several terms, which is fully described in Section~\ref{sec:unsupervised}. To encourage the optical flow to satisfy the epipolar geometry, one of the terms in $L$ is the one-sided epipolar error, given by the global geometric loss function
\begin{equation}
    L_{\text{e}} (\m V, \v \theta^{\star})  = \sum_{i = 1}^{N}\frac{(\v x_i'\tr\m E(\v \theta^{\star})\v x_i)^2}{[\m E(\v \theta^{\star})\v x_i]^2_1 + [\m E(\v \theta^{\star})\v x_i]^2_2}
    \label{eqn:geometric_loss}
\end{equation}
where the essential matrix $\m E$ is a function of its parameters $\v \theta^{\star}$ and $[\m E\v x]_j$ denotes the $j$\textsuperscript{th} component of the 3-vector $\m E\v x$. This error measures the sum squared distance of each point $\v x'$ to its corresponding epipolar line $\m E \v x$.
Note that $\v x_i$ is constant and $\v x_i' = \m K'^{-1} (\Dh{p}_i + \Dh{v}_i)$ is a function of the optical flow $\m V$.
The optimal essential matrix parameters $\v \theta^{\star}$ also depend on the optical flow $\m V$. We formulate essential matrix estimation as a lower-level optimization problem~\eqn{bilevel_constraint}, with a robust algebraic error objective function given by
\begin{align}
    l (\m V, \v \theta) &= \sum_{i=1}^{N}\rho\left(\v x_i'\tr \m E(\v \theta)\v x_i\right), \label{eqn:lower_loss}\\
    \rho(z; \delta) &= 
        \begin{cases}
        \frac{1}{2}z^2, & \text{if } |z| < \delta \\
       \frac{1}{2}\delta^2, & \text{otherwise.} \\
        \end{cases}
    \label{eqn:robust_l2_norm}
\end{align}
The robust function $\rho(\cdot)$ is a truncated $L_2$ penalty function with an inlier threshold $\delta$.

To solve the bi-level optimization problem~\eqn{bilevel_objective} within a deep learning context, we need to back-propagate gradients through the essential matrix optimization layer. During the forward pass, for each image pair we solve problem \eqn{bilevel_constraint} using iteratively re-weighted least squares (IRLS) to minimize the robust objective function $l$ \eqn{lower_loss}. However, this function is non-convex with many local minima. Hence, we first obtain a robust initial estimate of the essential matrix parameters using RANSAC \cite{fischler1981random} with the five-point algorithm \cite{li2006five, nister2004efficient}, which we have implemented as an efficient GPU routine.

During the backward pass we need to compute $\text{d} \v \theta^{\star} / \text{d} \m V$, which amounts to differentiating the $\argmin$ function. This can be achieved using implicit differentiation described below. As we will see, the gradient computation is agnostic to the method used to solve the lower-level problem, and only requires that a solution be found. Importantly, this means that we do not need to back-propagate through the specific steps of the optimization algorithm.

\subsection{Implicit Differentiation}
\label{sec:bilevel_implicit_differentiation}

Robustly estimating the essential matrix---via the lower-level optimization problem in \eqn{bilevel_constraint}---does not have an analytic solution and involves a non-differentiable RANSAC procedure to mitigate the effect of outliers. Obtaining the gradient would therefore not be possible using explicit differentiation or direct automatic differentiation.
However, since the objective function of the lower-level optimization problem in \eqn{bilevel_constraint} is twice-differentiable, we can compute the gradient of the $\argmin$ function using implicit differentiation knowing only the optimal solution (and not how it was obtained)~\cite{gould2016differentiating, samuel2009learning, domke2012generic, ochs2015bilevel}.
The key result, a special case of Dini's implicit function theorem~\cite[p19]{dontchev2014implicit} applied to the optimality condition $\ddinline{f(x, \v y)}{\v y} = \v 0$, is given below for completeness.
\begin{lemma}
\label{lm:implicit_diff}
(Gould \etal \cite{gould2016differentiating})
Let $f: \R \times \R^n \rightarrow \R$ be a continuous function with first and second derivatives. Set $\v g(x)$ to be a stationary point of $f(x, \v y)$ with respect to $\v y$,
for example $\v g(x) \in \argmin_{\v y \in \R^n} f(x, \v y)$,
and let the Hessian $f_{YY}(x, \v g(x))$ be nonsingular.
Then the vector derivative of $\v g$ with respect to $x$ is 
\begin{align}
\label{eqn:implicit_diff}
\frac{\mathrm{d}\v g}{\mathrm{d}x} &= -f_{YY}(x, \v g(x))^{-1} f_{XY} (x, \v g(x))
\end{align}
where $f_{YY}(x, \v y) \doteq \deldel{f(x, \v y)}{\v y}  \in \R^{n \times n}$ and $f_{XY}(x, \v y) \doteq \frac{\partial^2 f(x, \v y)}{\partial x \partial \v y}  \in \R^{n}$.
\end{lemma}
\begin{proof}
The derivative of $f$ with respect to $\v y$, evaluated at the stationary point $\v g(x)$, is zero by definition. The result follows by differentiating both sides with respect to $x$ using the chain rule, \ie,
\begin{align}
f_{Y}(x, \v g(x)) \doteq \dd{f(x, \v y)}{\v y} \Big|_{\v y = \v g(x)} &= \v 0\\
\ddx{} f_{Y}(x, \v g(x)) &= \v 0\\
\therefore \; f_{XY}(x, \v g(x)) + f_{YY}(x, \v g(x)) \ddx{\v g} &= \v 0
\end{align}
and rearranging the terms.
\end{proof}

For brevity we have shown the derivative with respect to a single parameter. For multiple parameters, the derivative can be computed with respect to each parameter separately.
Here, the matrix $f_{YY}$ need only be inverted or decomposed once for the full set of parameters.
It is also worth noting that the gradient is valid for any stationary point $\v g(x)$ of $f$, including local minima, maxima and saddle points.

Applied to the problem under consideration, we get
\begin{equation}
\label{eqn:implicit_diff_essential}
\dd{\v \theta^{\star}}{\m V} = -\left(\deldel{l(\m V, \v \theta^{\star}(\m V))}{\Theta}\right)^{-1} \frac{\partial^2 l(\m V, \v \theta^{\star}(\m V))}{\partial V \partial \Theta}.
\end{equation}
Automatic differentiation, such as the Autograd package in PyTorch, can be used to compute the necessary Jacobian and Hessian matrices.
Observe from~\eqref{eqn:implicit_diff_essential} that although we require $\v \theta^{\star}$ to be a stationary point of $l$, the computation of the gradient is independent of the algorithmic steps used to determine $\v\theta^{\star}$ and hence $\m E$. Thus to be clear, automatic differentiation, if used, is applied to the objective function itself and not the algorithmic procedure used to find its minimum, different from standard usage in deep learning models.
The total derivative of $L$ in problem \eqn{bilevel_objective} is then
\begin{equation}
\label{eqn:implicit_diff_derivative}
\dd{L(\m V, \v \theta^{\star}(\m V))}{\m V} = \del{L}{\m V} + \del{L}{\v \theta^{\star}} \dd{\v \theta^{\star}}{\m V}
\end{equation}
with all terms evaluated at $(\m V, \v \theta^{\star})$.
This defines the exact gradient of loss function $L$ constrained by \eqn{bilevel_constraint}. 

\section{Unsupervised Optical Flow Estimation}
\label{sec:unsupervised}

Now that we have a means of incorporating a global geometric loss function and a geometric estimation layer into an end-to-end learning framework, we can present our full training pipeline as shown in \fig{flowchart}.
At a high level, our network computes the optical flow from a pair of images and then estimates the camera motion using an embedded robust geometric optimization algorithm.
The estimated camera motion is used to self-supervise the optical flow network, alongside standard photometric, smoothness and consistency losses.
This approach can be used to enhance any state-of-the-art flow estimation network, which is currently the approach of Liu \etal \cite{liu2019selflow}.
As such, we use their unsupervised training strategy, which can be viewed as an intelligent data augmentation approach, in order to improve performance in highly occluded scenes.

\textbf{Loss functions:}
Similar to previous works, we use brightness constancy and local smoothness constraints by proposing photometric and smoothness losses. 
Following Meister \etal \cite{meister2018unflow}, we apply a ternary census transform $C(\cdot)$ on the input images, which is robust to real-world violations of the brightness constancy constraint \cite{zabih1994non, stein2004efficient}. 
Since brightness constancy does not hold for occluded pixels, we estimate an occlusion map based on the forward-backward consistency prior \cite{meister2018unflow} and only apply the photometric loss on non-occluded pixels.
For these pixels, we also apply a forward-backward consistency loss, to encourage consistent optical flow in both directions. 

Let $M_i$ indicate whether the $i$\textsuperscript{th} pixel is non-occluded and let $Z = \sum_{i=1}^N M_i$ be the number of non-occluded pixels.
We define the photometric loss $L_{\text{p}}$ and the forward-backward consistency loss $L_{\text{c}}$ as follows:
\begin{align}
    L_{\text{p}} &= \frac{1}{Z} \sum_{i=1}^{N} M_i \charb \left(C(I, \v p_i) - C(I', \v p_i + \v v_i) \right) \\
    L_{\text{c}} &= \frac{1}{Z} \sum_{i=1}^{N} M_i \charb \left(\m V^{\text{f}}(\v p_i) + \m V^{\text{b}}(\v p_i + \v v_i) \right)
\end{align}
where $\charb(\v z) = \frac{1}{n} \sum_{i=1}^{n} (z_i^2 + \epsilon^2)^{\gamma}$ is the element-wise average of the robust generalized Charbonnier penalty function \cite{sun2014quantitative} with $\epsilon = 10^{-3}$ and $\gamma = 0.45$,
and $\m V^{\text{f}}$ and $\m V^{\text{b}}$ are the predicted forward and backward optical flow, indexed by image space coordinates.

To encourage the predicted optical flow to be locally smooth, we apply an edge-aware smoothness loss \cite{godard2017unsupervised}, based on the assumption that motion boundaries often coincide with image edges. The smoothness loss is defined
\begin{equation}
    L_{\text{s}} = \frac{1}{2N} \sum_{i=1}^N \left( \sum_{a\in\{u, v\}}e^{-\frac{\alpha}{3} \left\|\del{I(\v p_i)}{a} \right\|_1} \left\|\del{\v v_i}{a} \right\|_1 \right).
\end{equation}

\textbf{Training strategy:}
We adopt the training strategy of Liu and co-authors \cite{liu2019selflow, liu2019ddflow} by having a teacher network and a student network in order to artificially generate occlusions and apply supervision on these regions.
The teacher network is first trained until convergence with the proposed loss functions.
The output of this network is then used to supervise the training of a student network, whose weights are initialized from the teacher network. 
Following Liu \etal \cite{liu2019selflow}, we generate occluded regions by computing SLIC superpixels \cite{achanta2012slic} and replacing randomly-selected superpixels with random noise.
Another source of generated occlusions comes from randomly cropping the input images \cite{liu2019ddflow}.
Pixels warped outside of the cropped image frame are considered to be occluded.
These artificial occlusions are only used in the student network, where the output of the teacher network is able to supervise the predicted flow.
This makes it possible for the student network to learn to estimate flow more accurately in occluded regions.

Let $O_i$ indicate the $i$\textsuperscript{th} pixel that is occluded in the synthetically-generated image but non-occluded in the original image, and let $Y = \sum_{i=1}^N O_i$ denote the number of such pixels.
We define the occlusion loss function for training the student network as
\begin{equation}
    L_{\text{o}} = \frac{1}{Y} \sum_{i=1}^{N} O_i \charb \left(\m V(\v p_i) - \tilde{\m V}(\v p_i) \right),
\end{equation}
where $\m V$ denotes the flow predicted by the student network and $\tilde{\m V}$ denotes the flow predicted by the teacher network. 

We also apply multi-scale training at five different resolutions to handle large motions. The epipolar loss is only applied at the highest resolution.
With weights $\lambda_\bullet$ on the loss terms, our total loss for the teacher network is
\begin{equation}
\label{eqn:loss_teacher}
L_\text{t} = \sum_{i=1}^5 \lambda^{i} \left( \lambda_{\text{p}}^{\vphantom{i}} L_{\text{p}}^{i} + \lambda_{\text{c}}^{\vphantom{i}} L_{\text{c}}^{i} +  \lambda_{\text{s}}^{\vphantom{i}} L_{\text{s}}^{i} \right) + \lambda_{\text{e}} L_{\text{e}},
\end{equation}
and the total loss for the student network is
\begin{equation}
\label{eqn:loss_student}
L_\text{s} = \sum_{i=1}^5 \lambda^{i} \left( \lambda_{\text{p}}^{\vphantom{i}} L_{\text{p}}^{i} + \lambda_{\text{c}}^{\vphantom{i}} L_{\text{c}}^{i} +  \lambda_{\text{s}}^{\vphantom{i}} L_{\text{s}}^{i} + \lambda_{\text{o}}^{\vphantom{i}} L_{\text{o}}^{i}\right) + \lambda_{\text{e}} L_{\text{e}}.
\end{equation}
Note that $L_\text{e}$ refers to the epipolar loss defined in \eqn{geometric_loss}. 

\section{Experiments}
\label{sec:experiments}

% http://pgfplots.sourceforge.net/pgfplotstable.pdf
\begin{table*}[!t]\centering
\caption{
Optical flow performance comparison on the KITTI 2012 and RGBD-SLAM datasets.
We report the mean End Point Error (EPE) of the predicted optical flow.
Parentheses and the suffix -ft indicate that the models were fine-tuned on the data, missing entries (--) indicate that the results were not reported, asterisks ($*$) indicate that the method uses the test set to select the best performing model, and daggers ($\dagger$) indicate that the results are pre-trained with the Ours-Baseline approach and finetuned with the losses proposed in the papers.}

\vspace{3pt}
\label{tab:flow}
\newcolumntype{C}{>{\centering\arraybackslash}X}
\setlength{\tabcolsep}{0pt} % Use/adjust if desired
\begin{tabularx}{0.17\textwidth}{@{}l l}
    \toprule
    \\
    \addlinespace
    & Method\\
    \midrule
    \multirow{5}{10pt}{\rotatebox{90}{\small Supervised}}
    &SpyNet-ft \cite{ranjan2017optical}\\
    &FlowNet2-ft \cite{ilg2017flownet}\\
    &PWC-Net \cite{sun2018pwc}\\
    &PWC-Net-ft \cite{sun2018pwc}\\
    &PWC-Net-ft* \cite{sun2018pwc}\\
    &SelFlow-ft \cite{liu2019selflow}\\
    \midrule
    \multirow{12}{10pt}{\rotatebox{90}{\small Unsupervised}}
    &UnsupFlownet \cite{jason2016back}\\
    &DSTFlow \cite{ren2017unsupervised}\\
    &DF-Net \cite{zou2018df}\\
    &UnFlow \cite{meister2018unflow}\\
    &OAFlow \cite{wang2018occlusion}\\
    &EPIFlow \cite{zhong2019unsupervised}\\
    &DDFlow \cite{liu2019ddflow} \\
    &SelFlow \cite{liu2019selflow} \\
    \cmidrule[\lightrulewidth](r{0.3em}){2-2}
    % \midrule
    &Ours-Baseline\\
    &Ours-Epipolar\\
    &Ours-Occlusion\\
    &Ours\\
    % &Ours-FT\\
    \bottomrule
\end{tabularx}%
\begin{filecontents*}{flow_w_test.csv}
4.13,4.1,0,2,0,0
1.28,1.8,0,1,0,0
4.14,0,0,0,4.836,5.406
1.45,1.7,0,0.9,1.222,6.705
0,0,0,0,3.64,5.05
0.76,1.5,0,0,0,0
11.3,9.9,4.3,4.6,0,0
10.43,12.4,3.29,4,0,0
3.54,4.4,0,0,0,0
3.29,0,1.26,0,0,0
3.55,4.2,0,0,0,0
2.51,3.4,0.99,1.3,5.161,6.537
2.35,3,1.02,1.1,5.006,6.509
1.69,2.2,0.91,1,4.949,6.472
3.488,0,1.24,0,5.444,6.894
2.613,0,0.99,0,4.819,6.328
1.966,0,1.193,0,4.949,6.472
1.562,1.9,0.938,1,4.632,6.12
\end{filecontents*}
\pgfplotstabletypeset[
    column type=, % clear defaults
    trim cells=true,
    begin table={\begin{tabularx}{0.83\textwidth}{@{}C C C C C C@{}}},
    end table={\end{tabularx}},
    col sep=comma,
    row sep=newline,
    header=false,
    every head row/.style={
        output empty row,
        before row={
            \toprule
            % \multicolumn{4}{c}{KITTI 2012}\\
            \multicolumn{2}{c}{KITTI 2012 (all)} & \multicolumn{2}{c}{KITTI 2012 (noc)} & \multicolumn{2}{c}{RGBD-SLAM}\\
            \cmidrule(lr){1-2}
            \cmidrule(lr){3-4}
            \cmidrule(lr){5-6}
            % \cmidrule(lr){6-6}
            train & test & train & test & validation & test\\
        },
        after row=\midrule,
    },
    every last row/.style={after row=\bottomrule},
    % every row no 1/.style={before row=\midrule},
    every row no 6/.style={before row=\midrule},
    every row no 14/.style={before row=\midrule},
    columns/0/.style={column name=,
        string replace={0}{}, % erase ’0’
        fixed zerofill,
        precision=2,
    },
    columns/1/.style={column name=,
        string replace={0}{}, % erase ’0’
        fixed zerofill,
        precision=2,
    },
    columns/2/.style={column name=,
        string replace={0}{}, % erase ’0’
        fixed zerofill,
        precision=2,
    },
    columns/3/.style={column name=,
        string replace={0}{}, % erase ’0’
        fixed zerofill,
        precision=2,
    },
    columns/4/.style={column name=,
        string replace={0}{}, % erase ’0’
        fixed, fixed zerofill,
        precision=2,
        % postproc cell content/.style={@cell content/.add={}{\%}},
    },
    columns/5/.style={column name=,
        string replace={0}{}, % erase ’0’
        fixed, fixed zerofill,
        precision=2,
        % postproc cell content/.style={@cell content/.add={}{\%}},
    },
    empty cells with={--}, % replace empty cells with ’--’
    %
    % For formatting individual cells:
    % Trained on same data:
    every row 0 column 0/.style={postproc cell content/.style={@cell content/.add={(}{)}}},
    every row 1 column 0/.style={postproc cell content/.style={@cell content/.add={(}{)}}},
    every row 3 column 0/.style={postproc cell content/.style={@cell content/.add={(}{)}}},
    every row 3 column 4/.style={postproc cell content/.style={@cell content/.add={\boldmath}{}}},
    every row 4 column 5/.style={postproc cell content/.style={@cell content/.add={\boldmath}{$^*$}}},
    every row 5 column 0/.style={postproc cell content/.style={@cell content/.add={(\boldmath}{)}}},
    every row 11 column 0/.style={postproc cell content/.style={@cell content/.add={(}{)}}},
    every row 11 column 2/.style={postproc cell content/.style={@cell content/.add={(}{)}}},
    every row 11 column 4/.style={postproc cell content/.style={@cell content/.add={}{$^\dagger$}}},
    every row 11 column 5/.style={postproc cell content/.style={@cell content/.add={}{$^\dagger$}}},
    every row 12 column 4/.style={postproc cell content/.style={@cell content/.add={}{$^\dagger$}}},
    every row 12 column 5/.style={postproc cell content/.style={@cell content/.add={}{$^\dagger$}}},
    every row 13 column 4/.style={postproc cell content/.style={@cell content/.add={}{$^\dagger$}}},
    every row 13 column 5/.style={postproc cell content/.style={@cell content/.add={}{$^\dagger$}}},
    % Best values:
    every row 3 column 3/.style={postproc cell content/.style={@cell content/.add={\boldmath}{}}},
    every row 5 column 1/.style={postproc cell content/.style={@cell content/.add={\boldmath}{}}},
    every row 17 column 0/.style={postproc cell content/.style={@cell content/.add={\boldmath}{}}},
    every row 17 column 1/.style={postproc cell content/.style={@cell content/.add={\boldmath}{}}},
    every row 13 column 2/.style={postproc cell content/.style={@cell content/.add={\boldmath}{}}},
    every row 13 column 3/.style={postproc cell content/.style={@cell content/.add={\boldmath}{}}},
    every row 17 column 3/.style={postproc cell content/.style={@cell content/.add={\boldmath}{}}},
    every row 17 column 4/.style={postproc cell content/.style={@cell content/.add={\boldmath}{}}},
    every row 17 column 5/.style={postproc cell content/.style={@cell content/.add={\boldmath}{}}},
    % every row 14 column 2/.style={postproc cell content/.style={@cell content/.add={\boldmath}{}}},
    % every row 14 column 0/.style={postproc cell content/.style={@cell content/.add={\boldmath}{}}},
    % every row 15 column 2/.style={postproc cell content/.style={@cell content/.add={\boldmath}{}}},
    % every row 15 column 1/.style={postproc cell content/.style={@cell content/.add={\boldmath}{}}},
    % every row 15 column 3/.style={postproc cell content/.style={@cell content/.add={\boldmath}{}}},
    % every row 4 column 3/.style={postproc cell content/.style={@cell content/.add={\boldmath}{}}},
    % Percentages:
    % every row 12 column 4/.style={postproc cell content/.style={@cell content/.add={}{\%}}},
    % every row 13 column 4/.style={postproc cell content/.style={@cell content/.add={\boldmath}{\%}}},
    % every row 14 column 4/.style={postproc cell content/.style={@cell content/.add={\boldmath}{\%}}},
    % every row 15 column 4/.style={postproc cell content/.style={@cell content/.add={\boldmath}{\%}}},
    % every row 12 column 5/.style={postproc cell content/.style={@cell content/.add={}{\%}}},
    % every row 13 column 5/.style={postproc cell content/.style={@cell content/.add={}{\%}}},
    % every row 14 column 5/.style={postproc cell content/.style={@cell content/.add={}{\%}}},
    % every row 15 column 5/.style={postproc cell content/.style={@cell content/.add={\boldmath}{\%}}},
]{flow_w_test.csv} % filename/path to file
\end{table*}

{\bf Datasets:} We evaluate our method for unsupervised optical flow estimation on two datasets: the standard KITTI 2012 Flow dataset~\cite{geiger2012we} and the more challenging RGB-D SLAM dataset~\cite{sturm2012benchmark}. The KITTI dataset~\cite{geiger2013vision} contains outdoor road scenes captured by a car-mounted stereo camera rig. Since we are only estimating rigid flow, we evaluate on the KITTI 2012 Flow subset~\cite{geiger2012we}, which has 194 training images with sparse ground-truth optical flow and 195 test images. We do not train on this data, instead we use the KITTI Visual Odometry (VO) dataset which has similar characteristics. 
This has 22 sequences with 87\,060 consecutive image pairs.
We leave out sequences 9 and 10 for motion evaluation and use the remainder for training. 
The RGB-D SLAM dataset~\cite{sturm2012benchmark} is an indoor SLAM dataset with ground-truth pose and depth, from which optical flow can be calculated. The dataset contains varied camera motions, many featureless regions, repetitive patterns, and motion blur, which are well-known to be challenging for optical flow estimation. We select all sequences from the Hand-held SLAM, Robot SLAM, and 3D Object Reconstruction categories. We set aside ``fr1/360'', ``fr2/360\_hemisphere'', ``fr2/pioneer\_360'', and ``fr3/teddy'' for testing. While most of these sequences are static, we remove those few frames that contain dynamic objects.
For the training data, we select image pairs with diverse flow ranges, randomly sampling equally from buckets with a maximum flow of 5--40 pixels, 40--80 pixels and 80--120 pixels. For the test data, we sub-sample the video frames such that the baseline between each image pair is at least 3cm. 
We obtain training and test sets of 29\,106 and 1\,667 image pairs.

{\bf Metrics:} For optical flow, we provide a comparison with a range of recent supervised and unsupervised methods using the standard Average End Point Error (AEPE) metric, given by $\text{AEPE} = \frac{1}{N} \sum_{i=1}^N \|\v v^{i} - \v v^{i}_{\text{gt}}\|$, where $\v v^{i}$ is the predicted flow at the $i$\textsuperscript{th} pixel, $\v v^{i}_{\text{gt}}$ is the ground-truth flow, and $N$ is the number of ground-truth pixels. 
For camera motion evaluation, we use the standard KITTI VO dataset evaluation criterion \cite{geiger2012we}, which evaluates on sub-sequences of length (100, 200, $\dots$, 800) meters, and report the average relative rotational and translational errors for the test sequences 9 and 10 in \tab{motion}.
Let $\m \Delta \m T_{ij} \in \mathrm{SE}(3)$ denote the delta pose difference between the estimated pose and the ground-truth pose given a pair of adjacent frames $i$ and $j$.
The delta pose difference is given by $\m \Delta \m T_{ij} = (\m T_{\text{gt},i}^{-1} \m T_{\text{gt},j}^{\vphantom{-1}} )^{-1} (\m T_i^{-1} \m T_j^{\vphantom{-1}})$, where $\m T_i$ and $\m T_j$ denote the poses at frames $(i, j)$.
The relative translation error is given by $t_{\text{err}}^{i} = \frac{1}{N} \sum_{ij} \|\mathrm{trans} (\m \Delta \m T_{ij})\|$ and relative rotation error is given by $r_{\text{err}} = \frac{1}{N} \sum_{ij} \arccos(0.5(\mathrm{trace}(\mathrm{rot}(\m \Delta \m T_{ij})) - 1)),$ where $\mathrm{trans(\cdot)}$ and $\mathrm{rot(\cdot)}$ extract the translation and the rotation parts of $\m \Delta \m T_{ij}$.

\subsection{Implementation Details} 
\label{sec:implementation}

We use PWC-Net~\cite{sun2018pwc} as our backbone network. Note however that our approach is network-agnostic so other optical flow networks can also be used. Multi-scale supervision is applied to capture large optical flows, especially on the RGB-D SLAM dataset. We generate a five-scale image pyramid starting at the original resolution then halving and warping at each successive level. In the original implementation of PWC-Net \cite{sun2018pwc}, a five-level pyramid of optical flow maps is predicted with the highest resolution being a quarter of the original image resolution. Therefore we scale the predicted optical flow by four using bilinear interpolation to match the corresponding image pyramid. For all networks, the weights for the multi-scale losses $(\lambda^{1},\lambda^{2},\lambda^{3},\lambda^{4},\lambda^{5})$ were set to $(1, 0.34, 0.31, 0.27, 0.08)$ as per Meister \etal \cite{meister2018unflow}, modulo a constant factor.

We train directly on the KITTI VO dataset and the RGBD SLAM dataset to obtain our baseline model. When training the baseline model on KITTI VO, we empirically set
$(\lambda_\text{p}, \lambda_\text{c}, \lambda_\text{s}, \lambda_\text{e})$ to $(1, 0.1, 0.1, 0)$.
We then add the epipolar loss with $\lambda_\text{e}$ set to $1000$ to fine-tune the teacher model.
When training the student model, we set $(\lambda_\text{p}, \lambda_\text{c}, \lambda_\text{s}, \lambda_\text{e}, \lambda_\text{o})$ to $(1, 0, 0, 1000, 1)$, adding the occlusion loss.
We used the same training strategy for the RGBD SLAM dataset, with
$(\lambda_\text{p}, \lambda_\text{c}, \lambda_\text{s}, \lambda_\text{e}, \lambda_\text{o})$ set to $(1, 0.1, 1, 100, 1)$.
All experiments were run on a PC with a single 11GB RTX 2080 Ti GPU.

\begin{figure*}[!t]\centering
	\begin{subfigure}[]{0.195\textwidth}\centering
		\includegraphics[width=\textwidth]{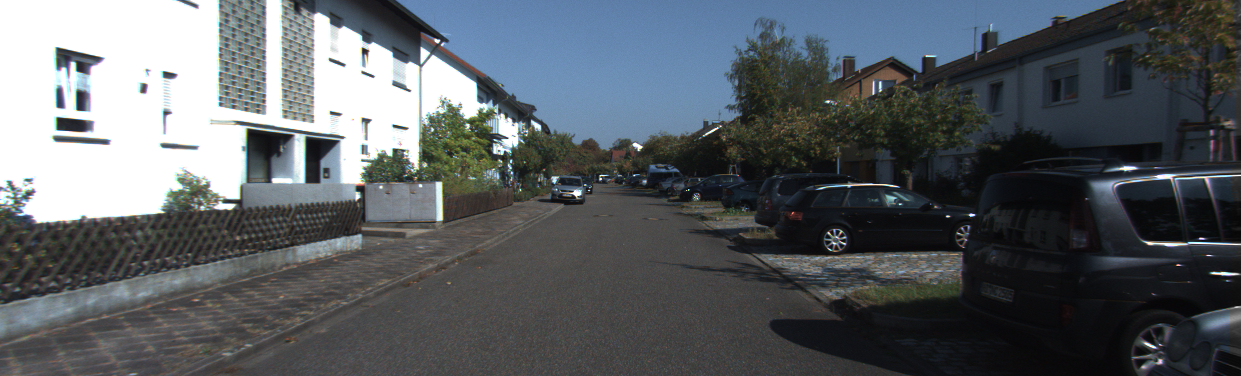}
	\end{subfigure}\hfill%
	\begin{subfigure}[]{0.195\textwidth}\centering
		\includegraphics[width=\textwidth]{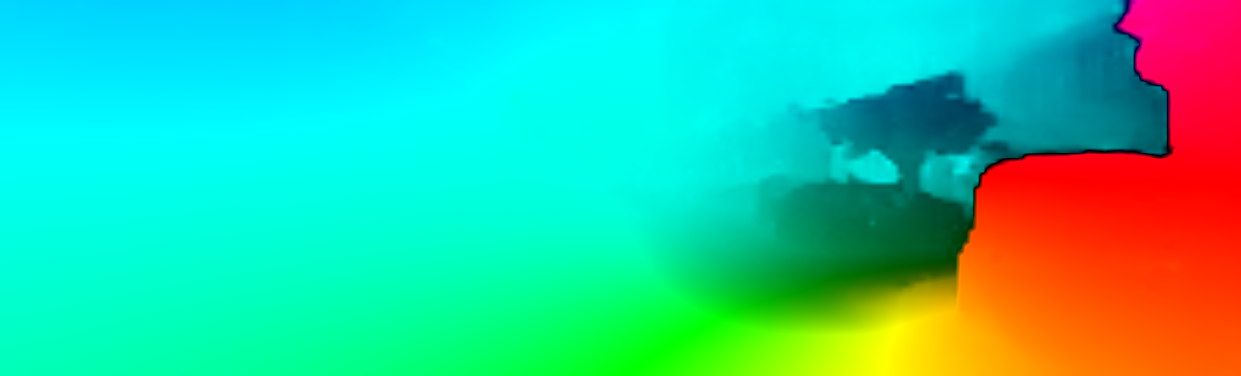}
	\end{subfigure}\hfill%
	\begin{subfigure}[]{0.195\textwidth}\centering
		\includegraphics[width=\textwidth]{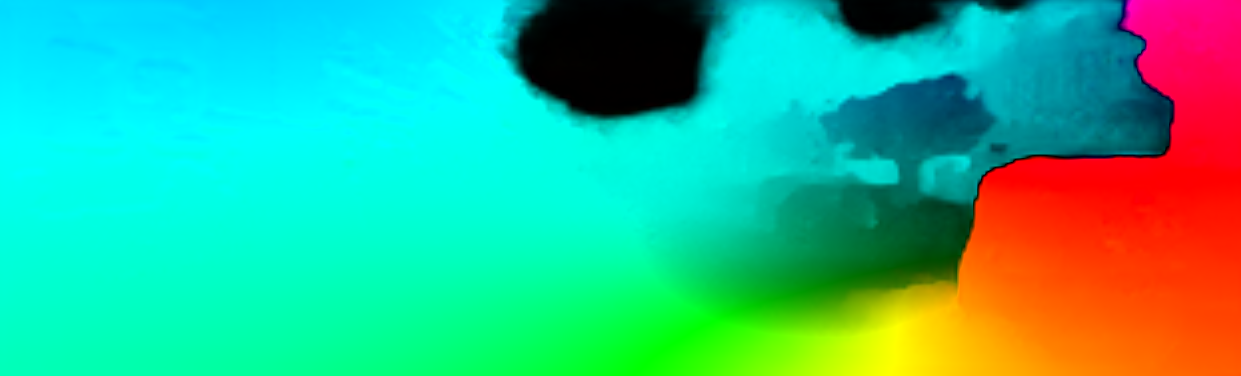}
	\end{subfigure}\hfill%
	\begin{subfigure}[]{0.195\textwidth}\centering
		\includegraphics[width=\textwidth]{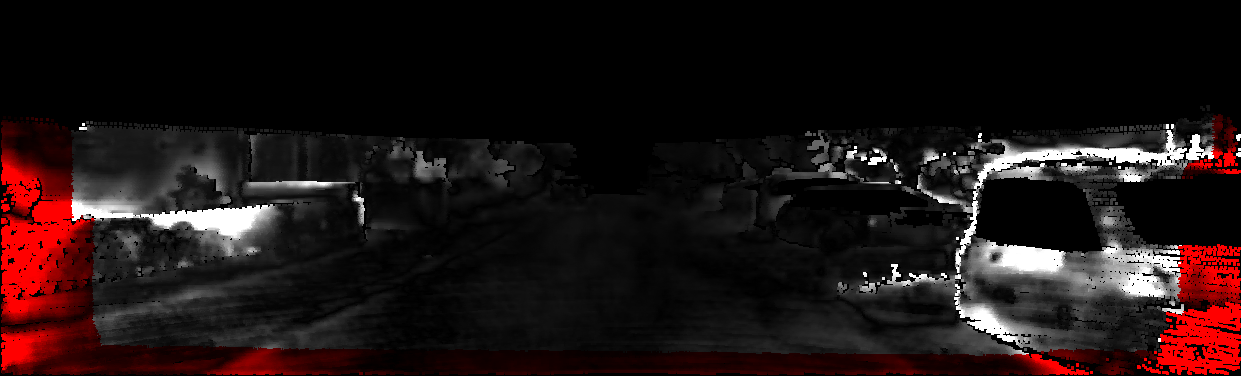}
	\end{subfigure}\hfill%
	\begin{subfigure}[]{0.195\textwidth}\centering
		\includegraphics[width=\textwidth]{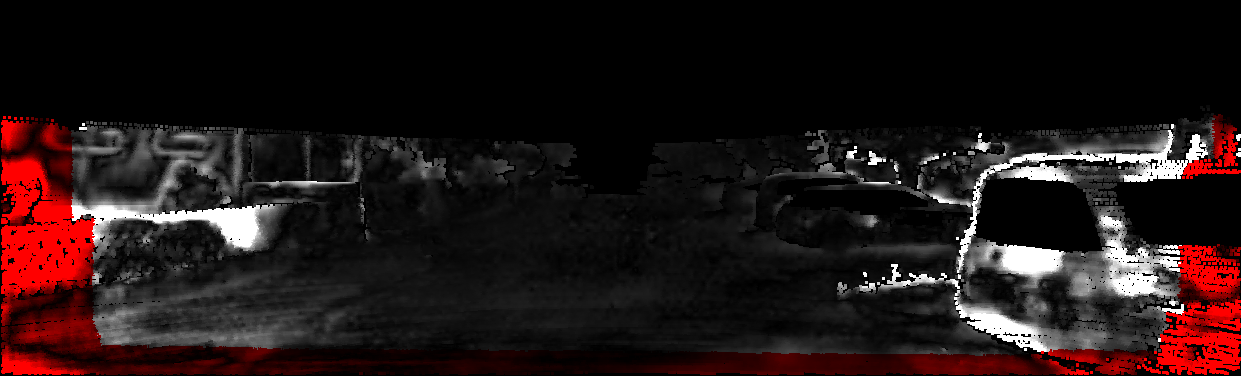}
	\end{subfigure}%
	\vspace{1pt}
	\begin{subfigure}[]{0.195\textwidth}\centering
		\includegraphics[width=\textwidth]{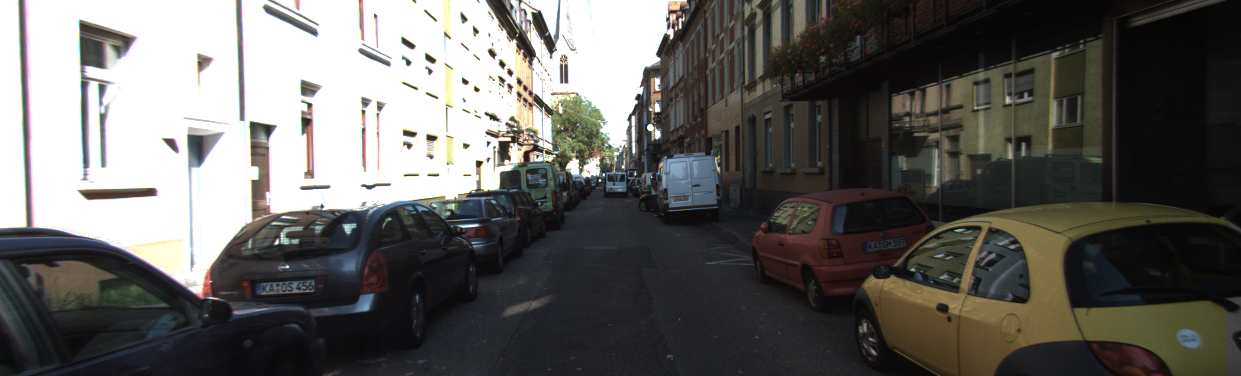}
		\subcaption{Input image}
		\label{fig:kitti_a}%
	\end{subfigure}\hfill%
	\begin{subfigure}[]{0.195\textwidth}\centering
		\includegraphics[width=\textwidth]{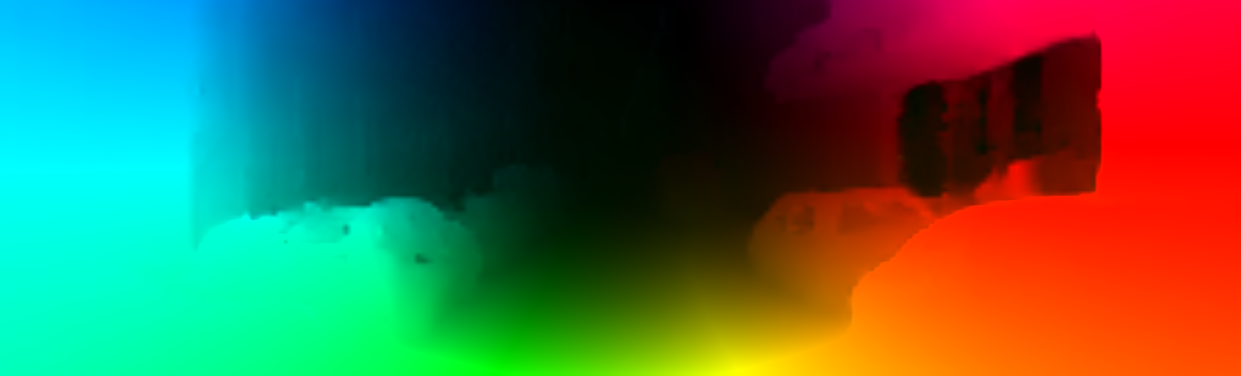}
		\subcaption{Our flow}
		\label{fig:kitti_b}%
	\end{subfigure}\hfill%
	\begin{subfigure}[]{0.195\textwidth}\centering
		\includegraphics[width=\textwidth]{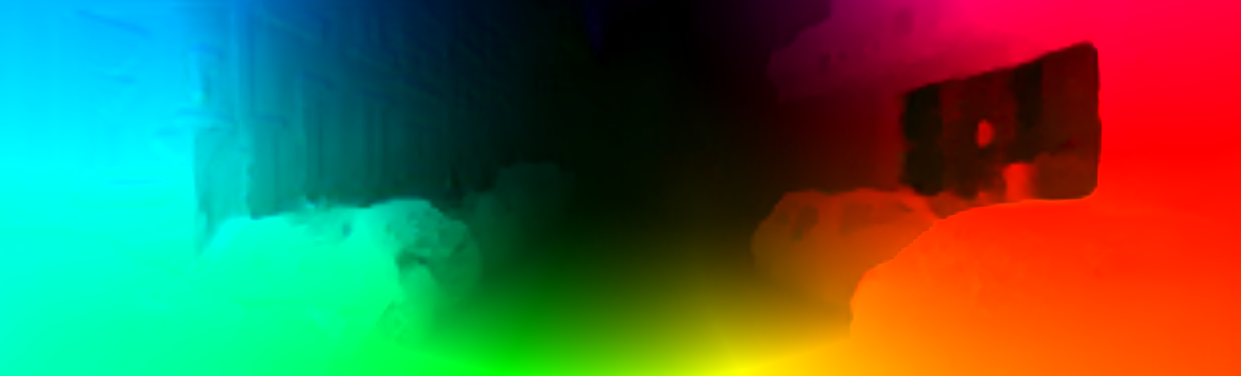}
		\subcaption{SelFlow \cite{liu2019selflow} flow}
		\label{fig:kitti_c}%
	\end{subfigure}\hfill%
	\begin{subfigure}[]{0.195\textwidth}\centering
		\includegraphics[width=\textwidth]{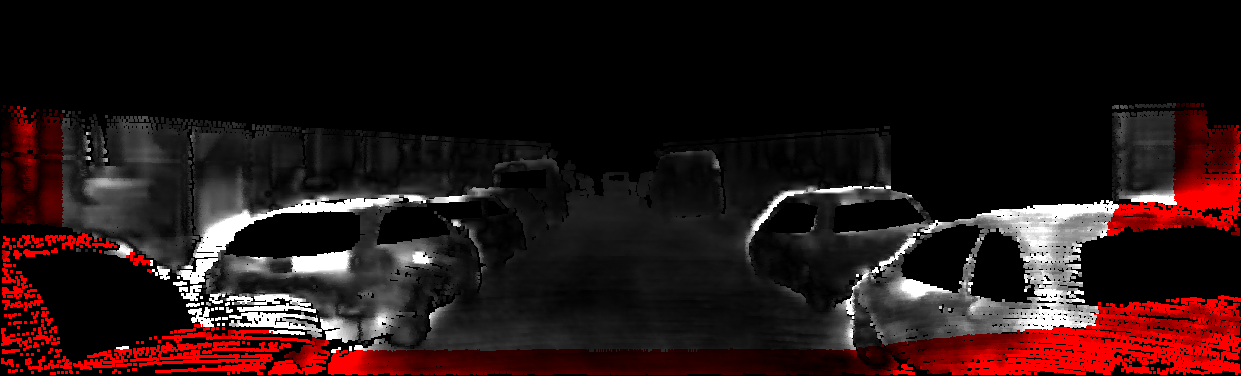}
		\subcaption{Our error}
		\label{fig:kitti_d}%
	\end{subfigure}\hfill%
	\begin{subfigure}[]{0.195\textwidth}\centering
		\includegraphics[width=\textwidth]{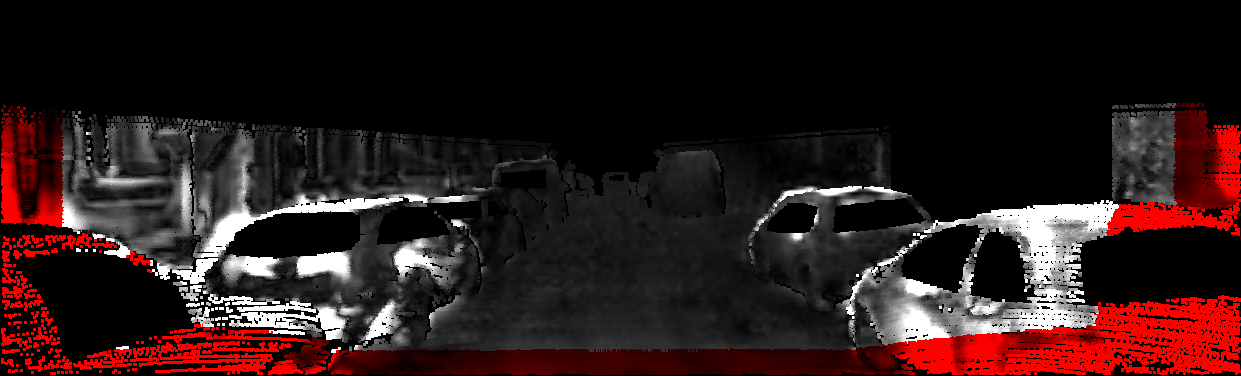}
		\subcaption{SelFlow \cite{liu2019selflow} error}
		\label{fig:kitti_e}%
	\end{subfigure}%
	\vspace{2pt}
	\caption{Qualitative results on the KITTI 2012 test set. We compare our method with SelFlow \cite{liu2019selflow}. \subref{fig:kitti_a} The first image of the input image pair. \subref{fig:kitti_b} Optical flow predicted by our model. \subref{fig:kitti_c} Optical flow predicted by SelFlow. \subref{fig:kitti_d} Optical flow prediction error with respect to the ground-truth using our method. \subref{fig:kitti_e} Optical flow prediction error using SelFlow.
	}
	\label{fig:kitti_qualitative}
\end{figure*}

\subsection{Essential Matrix Estimation}
\label{sec:essential_estimation}

For estimating the essential matrix, we first obtain a robust initial estimate using RANSAC \cite{fischler1981random} and the five-point algorithm \cite{li2006five, nister2004efficient}.
We randomly sample $10\,000$ correspondences from the predicted optical flow and each GPU thread randomly selects five points on which to run the five-point algorithm \cite{nister2004efficient}.
Hence, each thread provides an essential matrix hypothesis for RANSAC.
We select the essential matrix that has the most inliers, with respect to a set of $2\,000$ test correspondences, to initialize the iteratively re-weighted least squares (IRLS) algorithm.
IRLS iteratively minimizes our robust lower-level objective function $l$.
The stopping criteria is when the objective function is below $10^{-20}$ or the number of iterations exceeds $200$.
The inlier threshold $\delta$ was empirically set to $0.001$. 

\begin{figure*}[!t]\centering
    \begin{subfigure}[]{0.195\textwidth}\centering
		\includegraphics[width=\textwidth]{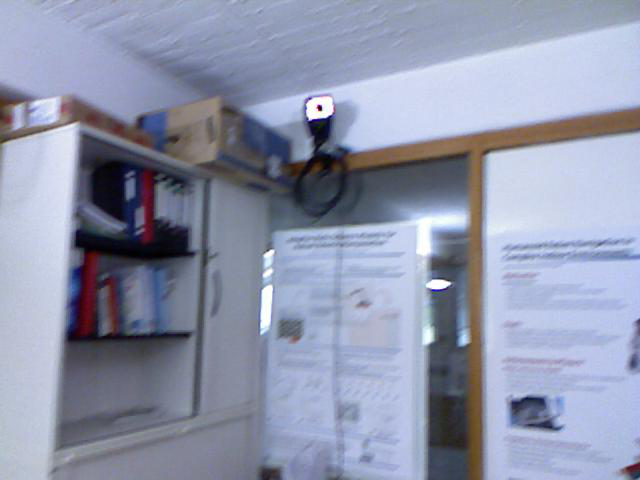}
	\end{subfigure}\vspace{1pt}\hfill%
	\begin{subfigure}[]{0.195\textwidth}\centering
		\includegraphics[width=\textwidth]{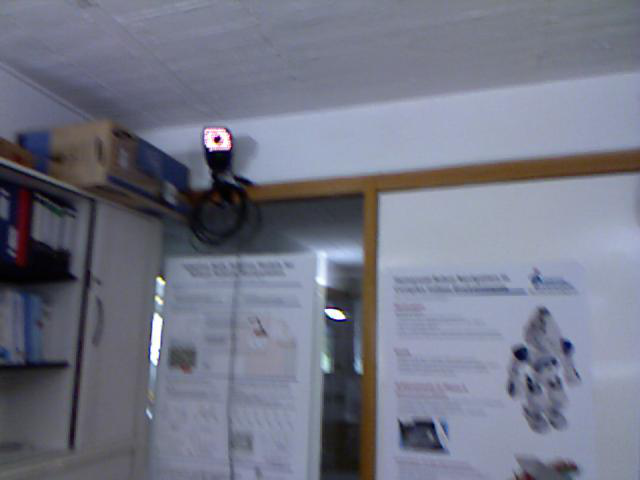}
	\end{subfigure}\hfill%
	\begin{subfigure}[]{0.195\textwidth}\centering
		\includegraphics[width=\textwidth]{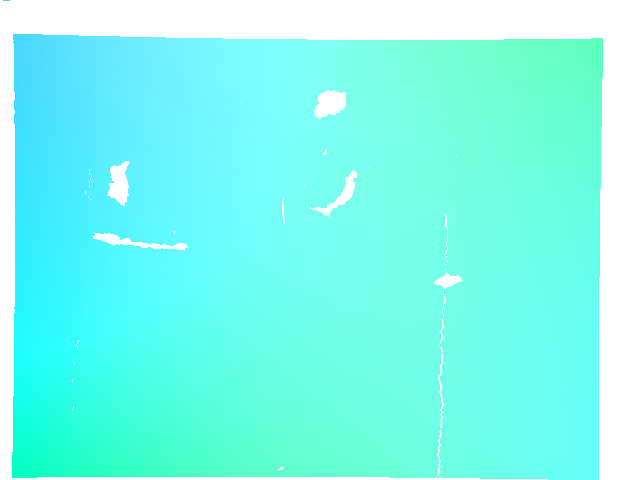}
	\end{subfigure}\hfill%
	\begin{subfigure}[]{0.195\textwidth}\centering
		\includegraphics[width=\textwidth]{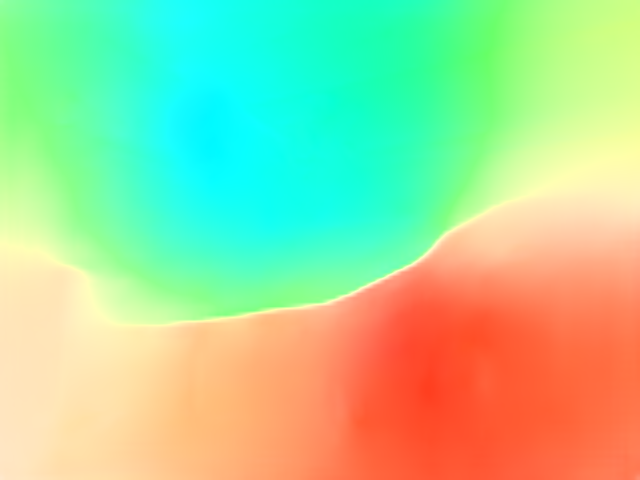}
	\end{subfigure}\hfill%
	\begin{subfigure}[]{0.195\textwidth}\centering
		\includegraphics[width=\textwidth]{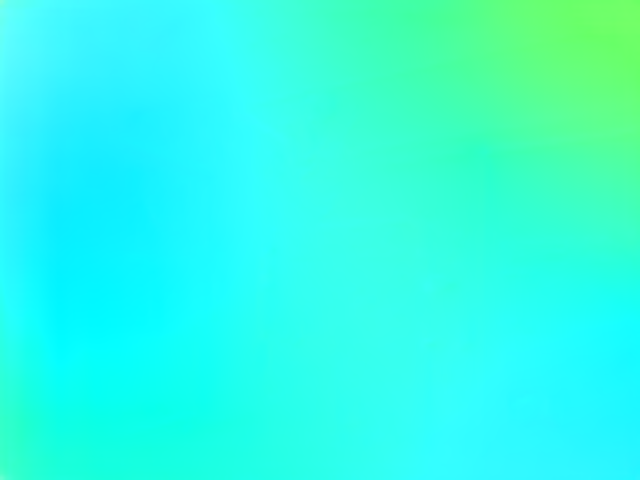}
	\end{subfigure}\hfill%
	
	\begin{subfigure}[]{0.195\textwidth}\centering
		\includegraphics[width=\textwidth]{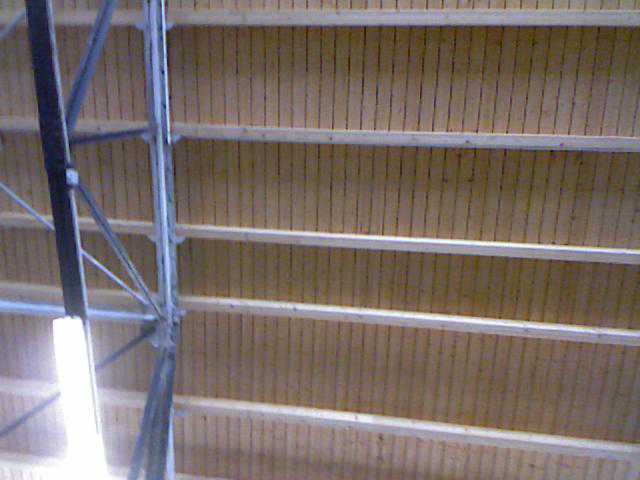}
		\subcaption{Input image 1}
		\label{fig:rgbd_a}%
	\end{subfigure}\hfill%
	\begin{subfigure}[]{0.195\textwidth}\centering
		\includegraphics[width=\textwidth]{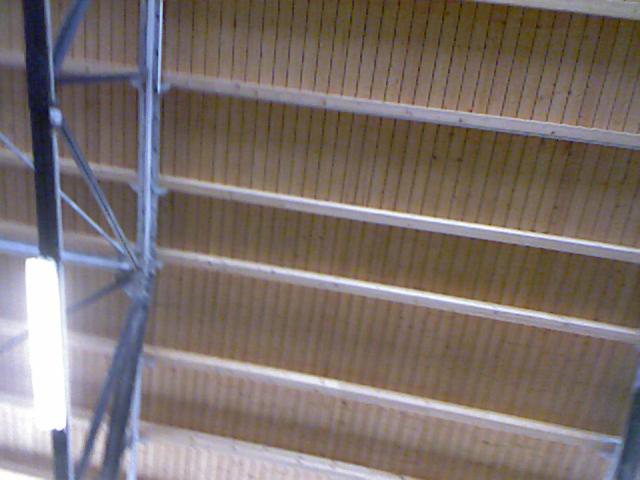}
		\subcaption{Input image 2}
		\label{fig:rgbd_b}%
	\end{subfigure}\hfill%
	\begin{subfigure}[]{0.195\textwidth}\centering
		\includegraphics[width=\textwidth]{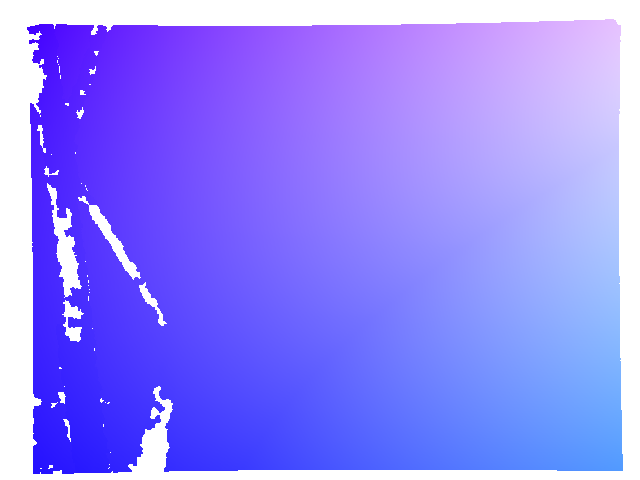}
		\subcaption{GT flow}
		\label{fig:rgbd_c}%
	\end{subfigure}\hfill%
	\begin{subfigure}[]{0.195\textwidth}\centering
		\includegraphics[width=\textwidth]{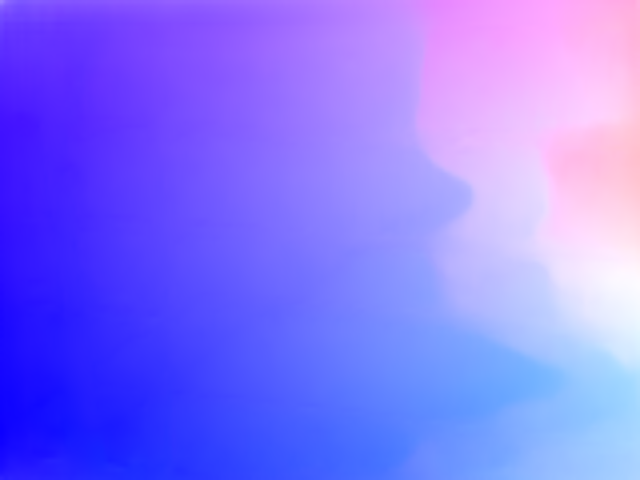}
		\subcaption{Our flow without $L_\text{e}$}
		\label{fig:rgbd_d}%
	\end{subfigure}\hfill%
	\begin{subfigure}[]{0.195\textwidth}\centering
		\includegraphics[width=\textwidth]{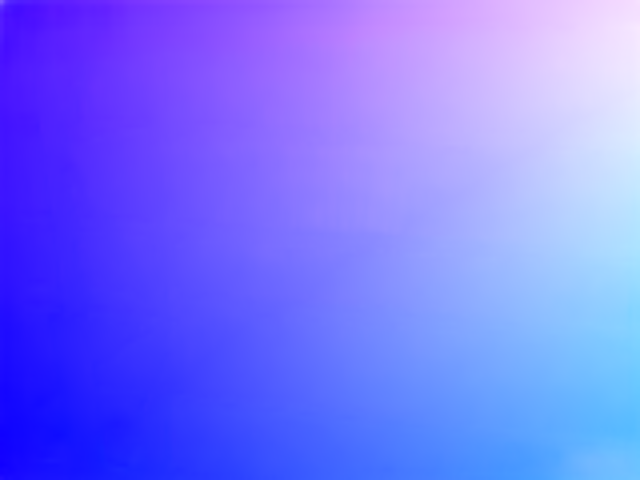}
		\subcaption{Our flow}
		\label{fig:rgbd_e}%
	\end{subfigure}%
	\vspace{2pt}
	\caption{Qualitative results on the RGBD-SLAM test set. The ground-truth flow is generated from the sparse ground-truth depth and the ground-truth pose provided by the dataset. The examples show that in cases of large motions, featureless regions and repetitive textures, the global geometric loss helps the network to learn to predict correct optical flow.}
	\label{fig:rgbd_qualitative}
\end{figure*}

\subsection{Optical Flow Results}
\label{sec:flow_results}

Quantitative results for optical flow on the KITTI and RGB-D SLAM datasets are reported in \tab{flow}. For ablation purposes, we also provide results for Ours-baseline, Ours-Epipolar, and Ours-Occlusion. Ours-baseline refers to the model that is only trained with $L_\text{p}$, $L_\text{c}$ and $L_\text{s}$. Ours-Epipolar refers to the model that is trained with the same losses but also the epipolar loss $L_\text{e}$. This is the result for our teacher network. Ours-Occlusion refers to training the baseline network with the teacher-student strategy described in \Sect{unsupervised}, but without our proposed epipolar loss.  

For KITTI, our model outperforms previous unsupervised optical flow methods and achieves results very close to supervised methods.
The ablation results indicate that the global geometric losses have a significant positive impact on the optical flow quality, decreasing error by approximately $20\%$ on average compared to the baseline. 
We can also see that the data augmentation technique proposed by Liu \etal \cite{liu2019selflow} only improves on the occluded pixels while our method improves on both the occluded and non-occluded pixels.
Note that compared with SelFlow~\cite{liu2019selflow}, our method only uses two frames as input for training and testing while SelFlow uses five frames for training and three frames for testing. We show improved performance despite using fewer frames as input. 

For the RGB-D SLAM dataset, our model outperforms previous state-of-the-art unsupervised optical flow methods and has slightly lower performance compared to a supervised method. 
Note that the results of other unsupervised methods reported in the table share the same backbone network as ours and are pre-trained using our baseline approach. They are then finetuned with the proposed losses in the respective papers.
The supervised method overfits on the training and validation data and therefore uses the test set to select the best performing model.
We show that we also make significant improvements over the baseline method on this challenging dataset, indicating the usefulness of the global geometric constraint. 
This dataset has a much wider variety of camera motions than the KITTI dataset, making optical flow estimation more difficult and camera motion estimation more helpful. 
Qualitative results for the RGB-D SLAM dataset are shown in \fig{rgbd_qualitative}. They demonstrate that the brightness constancy and smoothness assumptions are insufficient to correctly resolve the flow in challenging scenarios. 

\begin{table}[!t]
    \centering
    \caption{Odometry comparison for the KITTI VO dataset. We compare with an existing SLAM system and state-of-the-art unsupervised depth and motion learning algorithms. We report the translation error (\%) and the rotation error (degrees per 100m). }
    \vspace{2pt}
    \label{tab:motion}
    \resizebox{8.25cm}{!}{
    \begin{tabular}{lcccc}
    \toprule
    & \multicolumn{2}{c}{Seq. 9}  & \multicolumn{2}{c}{Seq. 10} \\
    \cmidrule(lr){2-3}
    \cmidrule(lr){4-5}
    Method & $t_{\text{err}}(\%) $ & $r_{\text{err}}(\degree/100\text{m}) $ & $t_{\text{err}}(\%) $ & $r_{\text{err}}(\degree/100\text{m}) $ \\
    \midrule
    ORB-SLAM \cite{mur2015orb} & $2.51$ & $0.26$ & $2.10$ & $0.48$ \\
    \midrule
    Zhou \etal \cite{zhou2017unsupervised} & $17.72$ & $6.82$ & $36.57$ & $17.69$ \\
    Zhan \etal \cite{zhan2018unsupervised} & $6.87$ & $3.60$ & $7.87$ & $3.41$ \\
    Gordon \etal \cite{gordon2019depth} & $\mathbf{3.10}$ & -- & $5.40$ & -- \\
    Bian \etal \cite{bian2019unsupervised} & $6.07$ & $2.19$ & $7.56$ & $4.63$ \\
    Ours & $4.36$ & $\mathbf{0.69}$ & $\mathbf{4.04}$ & $\mathbf{1.37}$ \\
    \bottomrule
    \end{tabular}}
\end{table}

\begin{figure*}[!t]\centering
    \begin{subfigure}[]{0.49\textwidth}\centering
        \includegraphics[width=\textwidth, trim=0 8pt 0 0, clip]{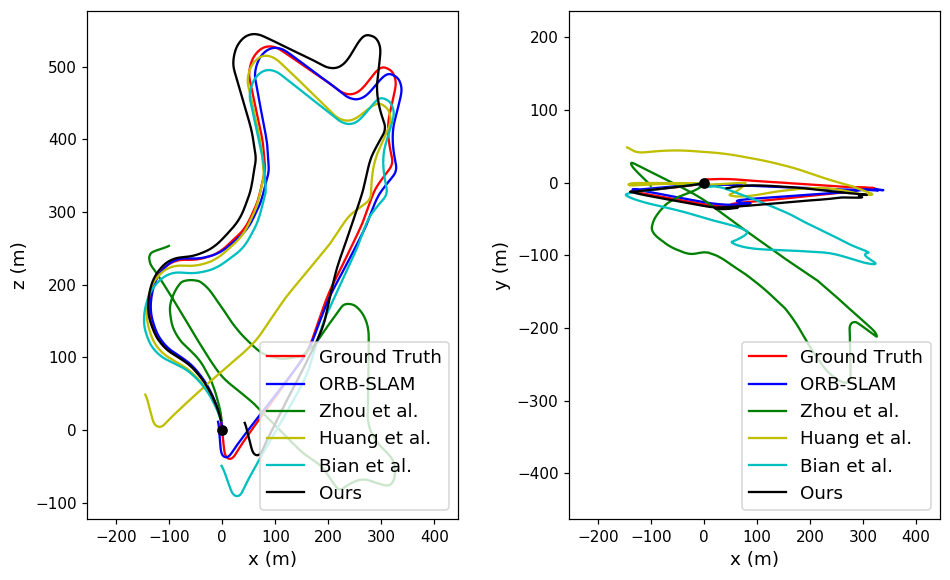}
        \subcaption{Seq. 9}
        \label{fig:seq09}
    \end{subfigure}
    \begin{subfigure}[]{0.50\textwidth}\centering
        \includegraphics[width=\textwidth, trim=0 8pt 0 0, clip]{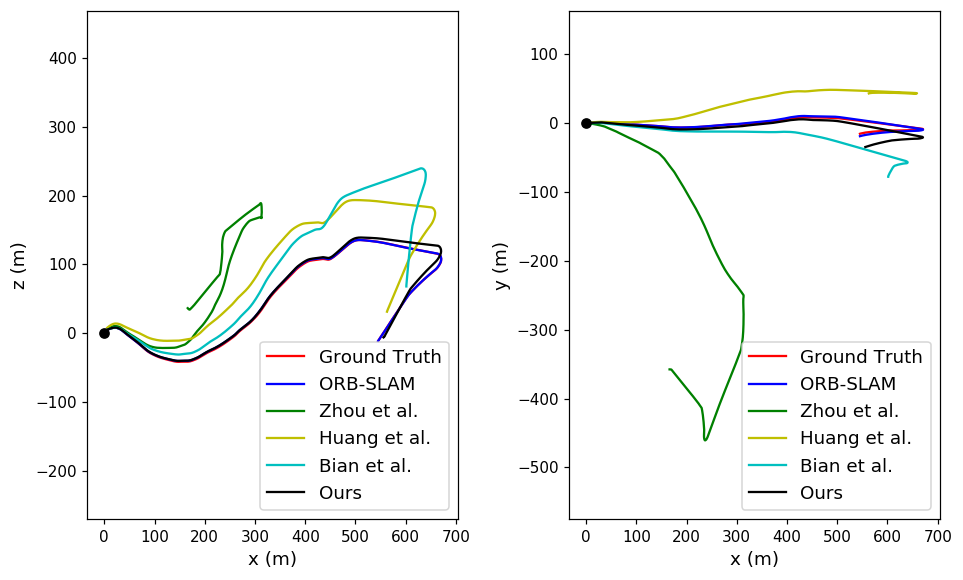}
        \subcaption{Seq. 10}
        \label{fig:seq09}
    \end{subfigure}
    \caption{Qualitative results on KITTI VO sequences 9 and 10. For each sequence we provide the xz and xy odometry map. Second to ORB-SLAM, our method is closest to the ground-truth odometry. When visualizing the xy odometry map, it can be seen that unsupervised learning methods \cite{zhou2017unsupervised, zhan2018unsupervised, bian2019unsupervised} have significant error in the y-axis direction, while our method has minimal error.}
    \label{fig:pose}
\end{figure*}

\subsection{Ego-motion Results}
\label{sec:motion_results}

We estimate the camera pose by decomposing our estimated essential matrix frame-by-frame, without any bundle adjustment, as opposed to ORB-SLAM \cite{mur2015orb}.
Since our estimated essential matrix does not contain any scale information, we have to align our scale with the ground-truth frame-by-frame.
For fair comparison, we also align the scales of the compared methods \cite{zhou2017unsupervised, zhan2018unsupervised, bian2019unsupervised}.

The quantitative results for ego-motion estimation on KITTI VO are shown in \tab{motion}.
Our method of estimating the camera pose significantly outperforms methods that directly regress the pose using a network.
Qualitative results for the odometry comparison are shown in \fig{pose}, from which we can see our algorithm achieves performance close to ORB-SLAM, up to scale.

\section{Discussion and Conclusion}
\label{sec:conclusion}

In this paper, we have proposed a pipeline that is able to learn optical flow and egomotion simultaneously in an unsupervised manner by incorporating global geometric constraints into an optical flow estimation network. 
In particular, our method uses the implicit differentiation technique to allow back-propagating the gradients through a complicated geometric estimation algorithm without needing to compute the gradient for each algorithmic step.
Given that the algorithm is complex, iterative, and involves a non-differentiable RANSAC procedure, it would otherwise be impossible to train end-to-end.
Unlike approaches such as differentiable RANSAC \cite{brachmann2017dsac}, we do not need to weaken any sub-components of the algorithm to make them easier for a network to compute, by for example, replacing the non-differentiable argmax hypothesis selction with probabilistic selection. 
Moreover, our formulation allows us to estimate the \emph{essential matrix} and back-propagate through this estimation layer. 
This gives a much tighter constraint than the fundamental matrix, having fewer degrees of freedom, and admits the use of state-of-the-art geometric algorithms.
In contrast, the 8-point algorithm for estimating the fundamental matrix linearizes the problem and does not exploit the given knowledge of the intrinsic camera parameters, which can result in "quite inaccurate" estimation of focal lengths and thus epipolar constraints \cite{hartley2002sensitivity}.
It has only been the algorithm of choice for network architectures because it has heretofore not been known how to back-propagate through an essential matrix estimation layer.

Our model produces state-of-the-art results for unsupervised learning of optical flow, including for challenging data on which existing algorithms are known to perform poorly. We have also demonstrated superior camera motion estimation by optimizing an essential matrix from the predicted optical flow, compared with unsupervised methods that directly regress camera pose.
Our approach to including a geometric estimation layer in a deep learning framework can be adapted to many other problems.
This work provides a case study that demonstrates the usefulness of implicit differentiation as a tool for improving computer vision models.

\paragraph{Acknowledgements} Shihao Jiang would like to thank Suryansh Kumar, Yiran Zhong, Yao Lu and Kartik Gupta for helpful discussions. This research is supported in part by the Australian Government through the Australian Research Council and Data61 CSIRO. 

{\small
\bibliographystyle{ieee_fullname}
\bibliography{egbib}
}

\end{document}